\documentclass{article}
\makeatletter
\let\asme@citex\@citex %
\makeatother

\usepackage{kr}

\makeatletter
\let\@citex\asme@citex %
\makeatother

\usepackage{times}
\usepackage{soul}
\usepackage{url}
\usepackage[normalem]{ulem}
\usepackage[hidelinks]{hyperref}
\usepackage[utf8]{inputenc}
\usepackage[small]{caption}
\usepackage{graphicx}
\usepackage{amsmath}
\usepackage{amsthm}
\usepackage{booktabs}
\usepackage{arydshln}
\usepackage{thmtools} 
\usepackage{thm-restate}

\usepackage{microtype}

\newtheorem{theorem}{Theorem}
\newtheorem{example}[theorem]{Example}
\newtheorem{definition}[theorem]{Definition}

\newtheorem{corollary}[theorem]{Corollary}
\newtheorem{proposition}[theorem]{Proposition}

\usepackage[switch,mathlines]{lineno}
\usepackage{etoolbox} %

\newcommand*\linenomathpatch[1]{%
  \cspreto{#1}{\linenomath}%
  \cspreto{#1*}{\linenomath}%
  \csappto{end#1}{\endlinenomath}%
  \csappto{end#1*}{\endlinenomath}%
}
\newcommand*\linenomathpatchAMS[1]{%
  \cspreto{#1}{\linenomathAMS}%
  \cspreto{#1*}{\linenomathAMS}%
  \csappto{end#1}{\endlinenomath}%
  \csappto{end#1*}{\endlinenomath}%
}

\expandafter\ifx\linenomath\linenomathWithnumbers
  \let\linenomathAMS\linenomathWithnumbers
  \patchcmd\linenomathAMS{\advance\postdisplaypenalty\linenopenalty}{}{}{}
\else
  \let\linenomathAMS\linenomathNonumbers
\fi

\linenomathpatch{equation}
\linenomathpatchAMS{gather}
\linenomathpatchAMS{multline}
\linenomathpatchAMS{align}
\linenomathpatchAMS{alignat}
\linenomathpatchAMS{flalign}

\title{Unifying approach to uniform expressivity of graph neural networks}

\iftrue
 \author{%
  Huan Luo
 \and
 Jonni Virtema\\
 \affiliations
 School of Computer Science, University of Sheffield, UK\\
 School of Computing Science, University of Glasgow, UK\\
 \emails
 \{jonni.virtema, huan.luo\}@glasgow.ac.uk
 }
 \fi

\usepackage{booktabs}
\usepackage{amsmath,amssymb}
\usepackage[colaction]{multicol}

\usepackage{mathtools}
\usepackage{xspace}
\usepackage{stmaryrd} %
\usepackage{enumerate}
\usepackage{thm-restate}

\newcommand{\cfo}{\mathrm{C}} 
\newcommand{\GML}{\mathrm{GML}}  %
\newcommand{\md}{\mathrm{md}}  %
\newcommand{\sd}{\mathrm{sd}}
\newcommand{\GN}{\mathcal{N}}  %

\newcommand{\Graphs}[1]{#1\text{-}\mathrm{Graphs}}  %

\newcommand{\AP}{\mathrm{AP}} %

\newcommand*{\ldblbrace}{\{\mskip-6mu\{}
\newcommand*{\rdblbrace}{\}\mskip-6mu\}}

\newcommand{\multiset}[1]{\ldblbrace#1\rdblbrace}

\newcommand{\dfn}\coloneqq

\newcommand{\gtmod}[2]{\langle #1 \rangle^{\ge #2}}

\newcommand{\T}{\mathcal{T}}
\newcommand{\R}{\mathbb{R}}
\newcommand{\N}{\mathbb{N}}
\newcommand{\agg}{\mathrm{agg}}
\newcommand{\comb}{\mathrm{comb}}
\newcommand{\cls}{\mathrm{cls}}
\newcommand{\emb}{\mathrm{emb}}

\newcommand{\col}{\mathrm{col}}
\newcommand{\hash}{\mathrm{HASH}}
\newcommand{\cb}{\mathrm{cb}}

\usepackage[most]{tcolorbox}
\newtcolorbox{mytodo}[2][]{
  arc=2mm,
  lower separated=false,
  enhanced,
  breakable,
  title={#2}, %
  #1 %
}

\newcommand{\jonni}[1]{\begin{mytodo}[colback=red!50!white,coltext=black]{Jonni}#1\end{mytodo}}

\usetikzlibrary{arrows.meta, positioning, shapes.geometric, calc}
\usepackage[T1]{fontenc}
\usepackage{subcaption}
\usepackage{natbib}

\begin{document}
\maketitle
\allowdisplaybreaks

\begin{abstract}
The expressive power of Graph Neural Networks (GNNs) is often analysed via correspondence to the Weisfeiler-Leman (WL) algorithm and fragments of first-order logic. Standard GNNs are limited to performing aggregation over immediate neighbourhoods or over global read-outs. To increase their expressivity, recent attempts have been made to incorporate substructural information (e.g. cycle counts and subgraph properties). In this paper, we formalize this architectural trend by introducing Template GNNs ($\T$-GNNs), a generalized framework where node features are updated by aggregating over valid template embeddings from a specified set of graph templates. We propose a corresponding logic, Graded template-modal logic ($\GML(\T)$), and generalized notions of template-based bisimulation and WL algorithm. We establish an equivalence between the expressive power of $\T$-GNNs and $\GML(\T)$, and provide a unifying approach for analysing GNN expressivity: we show how standard AC-GNNs and its recent variants can be interpreted as instantiations of $\T$-GNNs.
\end{abstract}

\section{Introduction}

The proliferation of structured data such as graphs and relational structures, and rapid development in the area of neural network machine learning in the past two decades, has led to the development of bespoke machine learning architectures for structured data, most notably Graph Neural Networks (GNNs) \citep{DBLP:journals/tnn/ScarselliGTHM09}.
Various extensions and variations of this model have been introduced and studied---however, from a high level perspective, a GNN iteratively and synchronously updates a vector of numerical features in every node of a graph by combining the node’s own feature vector with those of its neighbours \citep{gilmer2017neural}. In this level of abstraction, the GNN model is a variation of the local model of distributed computing of \cite{DBLP:journals/siamcomp/Linial92}.

Given an input graph, we want the GNN to output something meaningful. For example, this could be a property of the graph itself (e.g., decide whether the graph is Eulerian), a node property (e.g., inclusion to a dominating set), or a link prediction (e.g., predict whether any two nodes are path-connected). In this paper, we focus on the task of node classification where in the end a boolean-valued classification function is applied to the feature vector of each node.
In this framework, GNNs express unary (i.e., node-selecting) queries on graphs, which in the area of graph learning are known as node classifiers.

Due to intimate connections between the capabilities of GNNs, expressivity of query languages over graph data, modal logics, and methods related to the graph isomorphism problem, a growing community of researchers have revealed deep and precise connections between these seemingly disparate topics.

The (1-dimensional) Weisfeiler and Leman algorithm \citep{leman1968reduction}---a.k.a. colour refinement---is a well-known non-complete heuristic for deciding whether two graphs are isomorphic that is known to have the same (non-uniform) expressivity as message-passing GNNs \citep{morris2019weisfeiler,xu2019powerful}.

In the modal logic community, \cite{DBLP:journals/sLogica/Rijke00} developed the notion of graded bisimulation---relation between pointed labelled graphs---in order to characterise the expressivity of graded modal logic.  It is not hard to see that the notions of graded bisimulation and colour refinement are two sides of the same coin; two points in a graph are graded (k-)bisimilar if and only if the points are in the same colour class after applying the 1-WL algorithm (k rounds).

\cite{DBLP:journals/dc/HellaJKLLLSV15} discovered that there is a strong connection between the local model of distributed computing and graded modal logics. They established---utilising the notion of bisimulation---a one-to-one correspondence between local distributed algorithms and formulae of graded modal logic (in a non-uniform setting). \cite{DBLP:conf/nips/SatoYK19} further developed ideas from  \citeauthor{DBLP:journals/dc/HellaJKLLLSV15} and applied them to the GNN-setting.

The aforementioned characterisations of the expressivity of message passing models are all in some sense non-uniform. The connections to 1-WL relate to indistinguishability (with respect to any GNN of the given architecture), and the correspondence between local distributed algorithms and graded modal logic is formulated over degree bounded graphs. On the other hand, the uniform expressivity of a GNN architecture relates to studying the class of node classifiers that can be expressed by that GNN architecture (with respect to a full range of inputs that can be reasonably expected). The first characterisation of this kind was established by \cite{barcelo2020logical}---they established that a logical classier (i.e., one definable in first-order logic) is captured by an aggregate-combine-GNN (AC-GNN) if and only if it can be expressed in graded modal logic. They also established that each classifier definable in the two-variable fragment of first-order logic with counting quantifiers ($\cfo^2$) can be captured by a simple homogeneous aggregate-combine-readout-GNN (ACR-GNN)---leaving the converse, relative to logical classifiers, open.\looseness=-1

Since the seminal result of \cite{barcelo2020logical}, a plethora of similar results have been shown for different GNN architectures.
The connection between graded modal logic and AC-GNNs was further extended from logical classifiers to complete one-to-one correspondences by introducing mild forms of arithmetic to the logical side. \cite{DBLP:conf/icalp/BenediktLMT24} utilised logics with so-called \emph{Presburger quantifiers}, while \cite{DBLP:journals/theoretics/Grohe24} utilised counting terms and built-in relations.
  
\cite{grau2025correspondence} gave a comprehensive picture of correspondences between AC(R)-GNNs and various modal logics---the main differentiator of their approach to that of \citeauthor{barcelo2020logical} is the idea to replace the restriction to logical classifiers with the notion of bounded-GNNs (in the bounded setting, the ability of GNNs to differentiate multiplicities of elements in a multiset is limited to some constant). 

\cite{DBLP:journals/corr/abs-2508-06091} discovered that, actually, ACR-GNNs are strictly more expressive than $\cfo^2$, even when restricted to logical classifiers,  solving an open problem from  \citep{barcelo2020logical}. Their result indicates that, in general, when relating capabilities of GNNs to logics that are unable to do arithmetic, invariance under some bounded version of bisimulation is more appropriate restriction than restricting to logical classifiers. E.g., in the case of AC-GNNs, where a precise characterisation has been obtained in restriction to logical classifiers, invariance under graded bisimulation and FO-definability is equivalent to being invariant under bounded graded bisimulation \citep{otto2019graded}.

The provable limitations of the expressivity of aggregate-combine GNNs have led to various extensions of the GNN paradigm, as very simple tasks such as deciding graph-reachability or cycle detection are already beyond their capabilities.
Incorporating some form of recursion to the GNN formalism allows properties such as graph-reachability be expressible---typically non-uniform expressivity of GNN models of this type remain restricted to 1-WL, but their uniform expressivity is highly related to that of fixed-point logics such as the graded $\mu$-calculus \citep{BVVV25,pflueger2024recurrent,ahvonen2024logical}.

Two prominent paradigms for lifting GNN-expressivity beyond 1-WL
relate to enriching the graph features with subgraph counts \citep{DBLP:journals/pami/BouritsasFZB23,DBLP:conf/iclr/BevilacquaFLSCB22,DBLP:conf/nips/FrascaBBM22} or homomorphism pattern counts \citep{DBLP:conf/nips/BarceloGRR21,DBLP:conf/icml/JinBCL24}---a simple example here would be counting short cycles or paths, or homomorphisms to the complete graph of three vertices.
Yet, another branch of works has considered weaker GNN models; e.g, \citeauthor{DBLP:conf/kr/CucalaGMK23} 
(\citeyear{DBLP:conf/kr/CucalaGMK23,DBLP:conf/kr/CucalaG24}) related Max and Max-Sum GNNs to datalog.

Many of the characterisations of expressivity of various GNN-architectures use a similar recipe following the seminal result of \cite{barcelo2020logical}: First, propose an extension or variation of the AC-GNN model, define a variant of the WL-algorithm (or bisimulation) that corresponds to that GNN-model, and a modal logic that has modalities corresponding to the extension. One then needs to show that the bisimulation yields an upper bound for the expressivity of the GNN-model, and that every neural network yields a parameter $p$ such that $p$-bounded bisimulation is still an upper bound for its expressivity and has only finitely many $p$-bisimulation equivalence classes. Finally, one needs to show that every $p$-bisimulation equivalence class is definable in the newly defined logic and that every logical formula can be simulated by a GNN---the latter is often done by computing the truth value of each subformula recursively in the elements of the feature vectors.

While the recipe is simple to summarise, it does not mean that the generalisations are always easy to find or prove---sometimes only parts of the recipe can be applied. Some papers omit the logical counterpart completely and sometimes one gets a connection to a logic only by allowing infinitary connectives---this is the case when one cannot prove that the number of relevant bisimulation equivalence classes is finite. Particular examples utilising parts of the above recipe include \citep{BVVV25,pflueger2024recurrent,grau2025correspondence,soeteman2025logical,chen2025expressive}.

\vspace{2mm}
\noindent\textbf{Our contributions.}
In this paper, we introduce an abstract general model of graph neural networks that utilises aggregation over \emph{templates}. Templates are small patterns (i.e., graphs) that govern the form of message passing taking place in the GNN-computation. In standard AC-GNNs, messages are passed from a direct neighbour to another---hence the \emph{template} here is an edge. If on the other hand, messages are only allowed to be passed within triangles, the template would be a triangle.
Roughly speaking, aggregation in a template GNN using a template $T$ is over the multiset of ways to embed the template $T$ into the graph (having the node that aggregates mapped to a specified node in the template). A template GNN can have one or many templates that are used to aggregate simultaneously. For example, an AC-GNN that is enriched with triangle count information can be seen as a template GNN with two templates---an edge and a triangle.

After introducing  $\T$-GNNs, we define the corresponding notions of $\T$-WL-algorithm, graded $\T$-bisimulation, and graded template modal logic $\GML(\T)$.
We then establish that the expressivity of $\T$-GNNs is bounded by the graded $\T$-bisimulation, which also bounds the expressivity of $\GML(\T)$. Following the recipe of \cite{barcelo2020logical}, we show that every $\GML(\T)$-formula can be simulated by a $\T$-GNN. Finally, we establish a one-to-one correspondence between the uniform expressivity of $\T$-GNNs and $\GML(\T)$, in the bounded counting case introduced by \cite{grau2025correspondence} (here, in GNN-aggregation, multiplicities in multisets are capped with some constant).

Our main contribution is the formalisation of a general GNN-framework which can incorporate many of the modern paradigms for developing new expressive GNN-architectures.
Our main technical contribution, is a meta theorem yielding a family of theorems characterising the expressive power of GNNs. 
If one formalises their favourite GNN-model as a template GNN, our results give direct definitions for the corresponding WL-algorithm and logic that provably characterise the expressivity of the GNN model.

\section{Preliminaries}

\noindent\textbf{Graphs.} A \emph{labelled directed graph} is a tuple $G = (V, E, 
\lambda)$, where $V$ is a finite set of nodes,  $E \subseteq V \times V$ is the edge relation, and $\lambda : V \to \R^d$  assigns to each node a vector of real numbers of dimension $d$, for some fixed $d\in \N$. We call $d$ as the \emph{dimension} of $G$. $(G, v)$ is a pointed graph given a graph $G$ and a node $v \in V$.
Two pointed graphs are isomorphic (written as $(G, v) \cong (G', v')$) if there exists a bijection $f : V \to V'$ such that $f(v) = v'$, $\lambda(v) = \lambda'(f(v))$, for every $v\in V$, and $(v, u) \in E$ iff $(f(v), f(u)) \in E'$.
For a tuple $t=(a_1,\dots a_n)$, $t_i$ denotes its $i$th element $a_i$.
We identify assignments of type $\lambda : V \to \{0, 1\}^d$ with propositional assignments $\lambda' : V \to 2^{\{p_1,\dots,p_d\}}$, in the obvious way. That is, $p_i \in \lambda'(v)$ if and only if $\lambda(v)_i = 1$. Such labelled directed graphs can be treated as Kripke structures with the set of propositions $\{p_1,\dots,p_d\}$, and vice versa.

For a set $S$, we set $\Graphs{S}$ to be the class of labelled directed graphs that are labelled with elements from $S$.

\vspace{2mm}
\noindent\textbf{Graph classifiers and transformations.}
Let $S$ and $C$ be sets. A \emph{node $C$-classifier} for $\Graphs{S}$ is a function $f\colon \Graphs{S} \to \Graphs{C}$ that maintains the underlying graph structure (i.e., $G$ and $f(G)$ may differ only on $\lambda$). If $C=\{0,1\}$, this function is called a \emph{Boolean node classifier}. If $C=S$, we call this function as \emph{S-graph transformation}.

\vspace{2mm}
\noindent\textbf{GNN node classifiers.}  A standard aggregate-combine graph neural network (AC-GNN) consists of $L$ AC layers and a  classification function $\cls$, for $L\in\N$.
Intuitively, each GNN-layer of input dimension $d$ computes an $\R^d$-graph transformation, while the classification function computes a node $C$-classifier, for a finite set of classes $C$. The node classifier computed by the GNN is the composition of these functions.
Formally, an AC layer of input dimension $d$ is a pair $(\agg, \comb)$, where $\agg$ is an aggregation function mapping multisets of vectors of dimension $d$ to a vector of dimension $d$, and $\comb\colon \R^{2d} \to \R^d$ is a combination function mapping two vectors of dimension $d$ to a vector of dimension $d$.
At the $l$-th layer, the node label vector $\lambda^l(v)$ is updated via $comb\bigl(\lambda^{l-1}(v), \agg(\multiset{\lambda^{l-1}(u)}_{u \in N(v)})\bigr)$, where $N(v) = \{u \mid (v, u) \in E\}$ is the set of neighbours of $v$.
Each node $v \in V$ is classified according to a classification function $\cls\colon \R^d \to C$ applied to the final node label $\lambda^L(u)$. 
An application $\mathcal{N}(G, v)$ of an AC-GNN $\mathcal{N}$ to a pointed labelled graph $(G,v)$ is the value $\cls(\lambda^L(v))$. %

\vspace{2mm}
\noindent\textbf{Logical classifiers.} Consider a finite set of propositions $\AP$.\footnote{We restrict to finite $\AP$ to obtain a correspondence between $\AP$ and feature vectors of finite dimension $\lvert \AP \rvert$.} Formulae of basic modal logic (ML) are given by
\[
\varphi := p \mid \neg\varphi \mid  \varphi \wedge \varphi \mid \Diamond\varphi, \quad \text{where $p\in \AP$.}
\]
 Formulae are evaluated over pointed $\{0,1\}^{\lvert \AP\rvert}$-labelled graphs $(G,v)$ in the usual manner.
Graded modal logic (GML) extends ML with graded modalities $\Diamond^{\geq c}$, for positive integers $c$, where
\[
(G, v) \models \Diamond^{\geq c} \varphi \text{ iff } \lvert \{u \mid (v, u) \in E ~\text{and}~(G, u)\models \varphi\} \rvert \geq c.
\]
A (Boolean) \emph{GNN classifier $\mathcal{N}$ captures a logical classifier $\varphi$} if for every graph $G$ and node $v$ in $G$, it holds that $\mathcal{N}(G, v) = 1$ if and only if $(G, v) \models \varphi$.

\vspace{2mm}
\noindent\textbf{Weisfeiler-Leman (WL) algorithm.} The WL algorithm is an efficient heuristic originally introduced for checking graph isomorphism \citep{leman1968reduction}. Given a graph and a countably infinite set of colours $C$, the 1-dimensional WL (1-WL) test is a colour refinement algorithm that updates node colouring according to the following rule: $C^l(v) = \mathrm{HASH} \bigl(C^{l-1}(v), \multiset{ C^{l-1}(u)}_{u \in N(v)}\bigr)$, where $\mathrm{HASH}$ is assumed to be a perfect hash function. The algorithm terminates when the colouring is stable or the pre-specified number of iterations is exceeded. Two pointed graphs are said to be WL-indistinguishable if their final colour multisets match. It has been proven that GNNs and 1-WL have the same distinguishing power in the sense that a GNN can distinguish two nodes of a graph if and only if the colour refinement procedure assigns different colours to the two nodes \citep{xu2019powerful, morris2019weisfeiler}.

\section{Template GNNs}
In this section, we define our new model---template graph neural networks---and discuss how existing GNN-models can be interpreted as template GNNs. We start by introducing the notion of a template.

\begin{definition}[Template]
A \emph{template} is a structure $T = (V, E^+, E^-, r)$ such that
\begin{itemize}
	\item $V$ is a finite set of vertices,
	\item $E^+ \subseteq V \times V$ is a set of edges, where
    $(u, v) \in E^+$ is interpreted to mean that there is an edge from $u$ to $v$,
	\item $E^- \subseteq V \times V$ is a set of non-edges such that $E^+ \cap E^- = \emptyset$,
	\item $r \in V$ is a node called the root of the template.
\end{itemize}
Note that, we do not presume that $E^+ \cup E^- = V \times V$.
\end{definition}
A \emph{labelled template is a pair} $(T,\lambda)$, where  $T= (V, E^+, E^-, r)$ is a template and $\lambda:V \to \R^n$ is a labelling.

\begin{definition}[Template embedding]
A template embedding of $T = (V, E^+, E^-, r)$ into a pointed graph $(G,w) = (V_G, E_G,w)$ is an injective homomorphism $f\colon V \to V_G$ where $f(r)=w$ and for all $u, v \in V$ it holds that: 
\begin{enumerate}
	\item if $(u, v) \in E^+$ then $(f(u), f(v)) \in E_G$,
	\item if $(u, v) \in E^-$ then $(f(u), f(v)) \notin E_G$. 
\end{enumerate} 
Note that, if $E^+ \cup  E^- = V \times V$, a template embedding is simply an injective strong homomorphism from $(V, E^+, r)$ to $(G,w)$.
We  write $\emb(T, (G,w))$ for the set of all template embeddings of $T$ into $(G,w)$.
\end{definition}

A \emph{(multiset) aggregation function} is any function mapping multisets of real numbers to a real number. 

\begin{definition}[Template isomorphism]
Two labelled templates $(V_1, E_1^+, E_1^-, r_1,\lambda_1)$ and $(V_2, E_2^+, E_2^-, r_2, \lambda_2)$ are \emph{template isomorphic} if there is a bijection $f\colon V_1 \to V_2$ s.t. 
\begin{enumerate}
    \item $f(r_1) = r_2$,
	\item $(u, v) \in E_1^+$ if and only if $(f(u), f(v)) \in E_2^+$,
	\item $(u, v) \in E_1^-$ if and only if $(f(u), f(v)) \in E_2^-$,
    \item $\lambda_1(u)= \lambda_2(f(u))$, for every $u\in V_1\setminus\{r_1\}$.
\end{enumerate}
\end{definition}

\begin{definition}[Template aggregation function]
Let $T = (V, E^+, E^-, r)$ be a template and $d\in\N$. A function $\agg_T \colon (T,\lambda) \to \R^{d}$ that maps labelled $T$-templates to $\R^{d}$ is a \emph{$T$-aggregate function} if the function is invariant under template automorphisms. That is,
\(
\agg_T (T,\lambda_1) = \agg_T (T, \lambda_2),
\)
whenever $(T,\lambda_1)$ and $(T,\lambda_2)$ are template isomorphic.
\end{definition}

Intuitively, a template GNN is a message passing GNN, where messages are not passed via edges, but instead via template embeddings. For a fixed template $T$, the multiset of messages that a node $v$ receives, is the multiset of labelled graphs obtained via template embeddings. The multiset of labelled graphs is then transformed into a multiset of feature vectors by the template aggregate function. This multiset of feature vectors is then combined with the feature vector of $v$ as in standard AC-GNNs---via further aggregate/combine functions. The formal definition is given below.
\begin{definition}[Unary Template GNN]
Fix $L,d\in \N$ and a set of templates $\T$. A \emph{unary $L$-layer $\T$-GNN} (with feature dimension $d$) is a tuple
\[
\mathcal{N} = (\{\agg^l_{T^l}\}_{l\leq L}, \{\agg^l\}_{l\leq L}, \{\comb^l\}_{l\leq L},  \cls),
\]
where, for each layer $1\leq l\leq L$, $T^l\in \T$ is a template, $\agg^l_{T^l}$ is a $T^l$-aggregate function, $agg^l$ is an aggregate function, $\comb^l$ is a combination function, and $\cls : \R^d \to \{0, 1\} $ is a final Boolean classification function. 
At the $l$-th layer, the node feature vector (of dimension $d$) of a node $v$ is updated by 
\begin{multline*}
\lambda^l(v) \dfn \comb^l\big(\lambda^{l-1}(v),\\
\agg^l \multiset{\agg^l_{T^l} (T^l, \lambda^{l-1}_f) \mid f \in \emb(T^l, (G,v))}\big),
\end{multline*}
where $\lambda^{l-1}_f(u):= \lambda^{l-1}(f(u))$, for $u \in T^l$.
Finally, each node $v$ is classified by applying $\cls$ to the final node label $\lambda^L(u)$. 
When $\T=\{T\}$, we write $T$-GNN instead of $\T$-GNN. If $\GN$ is a $T$-GNN for some $\T$, we call it a \emph{template GNN}.
\end{definition}

We also consider template GNNs that aggregate over multiple templates simultaneously. We call these GNNs $n$-ary.

\begin{definition}[n-ary Template GNN]
Fix $L,d\in \N$ and a set of templates $\T$. An \emph{$n$-ary $L$-layer $\T$-GNN} (with feature dimension $d$) is a tuple
\[
\mathcal{N} = (\{\agg^l_{T_j^l}\}_{\substack{l\leq L \\ j\leq n}}, \{\agg_j^l\}_{\substack{l\leq L,\\  j\leq n}}, \{\comb^l\}_{l\leq L},  \cls),
\]
where, for each layer $1\leq l\leq L$ and $1\leq j\leq n$, $T_j^l\in \T$ is a template, $\agg^l_{T_j^l}$ is a $T_j^l$-aggregate function, $\agg_j^l$ is an aggregate function, $\comb^l$ is a combination function, and $\cls : \R^d \to \{0, 1\} $ is a final Boolean classification function.
At the $l$-th layer, the feature vector of node $v$ is updated by 
\begin{multline*}
\lambda^l(v) \dfn \comb^l\big(\lambda^{l-1}(v),\\
\agg_1^l \multiset{\agg^l_{T_1^l} (T_1^l, \lambda^{l-1}_f) \mid f \in \emb(T_1^l, (G,v))},\\
\vdots
\\
\agg_n^l \multiset{\agg^l_{T_n^l} (T_n^l, \lambda^{l-1}_f) \mid f \in \emb(T_n^l, (G,v))}\big),
\end{multline*}
where $\lambda^{l-1}_f$ is defined as in the unary case.
\end{definition}

For $c\geq 1$, we call a function $f$ whose parameters are (tuples of) multisets \emph{$c$-bounded}, if  $f(A)=f(A\upharpoonright c)$, where $A\upharpoonright c$ is obtained from $A$ by changing multiplicities larger than $c$ to $c$.
We call a $\T$-GNN $\GN$ \emph{$c$-bounded}, if all its outer aggregate functions\footnote{Template aggregate functions are trivially $c$-bounded, since their input is a single labelled graph.} are c-bounded, and we call  $\GN$ \emph{bounded}, if it is $c$-bounded for some $c$. For multisets $A$ and $B$ and $c\geq 1$, we write $A=_{ c}B$ if $A\upharpoonright c = B\upharpoonright c$.

\vspace{2mm}
\noindent\textbf{Families of GNNs interpreted as template GNNs.}
A standard AC-GNN (see e.g. \cite{barcelo2020logical}) can be interpreted as a unary template GNN with the template $T_1$, as shown in Figure \ref{fig:templates}(\subref{fig:t1}), which has vertex set  $V = \{r_1, a\}$, edge set $E_1^+ = \{(r_1, a)\}$, and non-edge set $E_1^- = \emptyset$. The number of valid template embeddings at node $v$ is the same as the number of neighbouring nodes of $v$. When $\agg_T$ 
projects to the feature vector of $a$ and the outer aggregation function is the same as the one used in the AC-GNN, the update rules of node feature vectors for $T$-GNNs and AC-GNNs coincide.

\cite{grau2025correspondence} introduced the family of bounded GNNs, where aggregation (and readout) functions are restricted to bounded functions.  The paper established that bounded GNNs with AC+ layers (defined below) have the same expressive power as the two variable fragment of first-order logic with counting quantifiers. For an AC+ layer, the node feature vectors $\lambda^l(v)$ is updated via
\begin{multline*}
\comb\big(\lambda^{l-1}(v),\\
\agg\multiset{\lambda^{l-1}{(u)}}_{u \in N(v)}, \agg\multiset{\lambda^{l-1}{(u)}}_{u \in \bar{N}(v)}),
\end{multline*}
where $\bar{N}(v) = \{u \mid (v, u) \notin E\} \setminus \{v\}$ is the set of non-neighbours of $v$ excluding $v$ itself. Bounded AC+-GNNs can be interpreted as a binary template GNNs with $\T = \{T_1, T_2\}$ for each layer, where $T_1$ is the same as above, and $T_2$ as shown in Figure \ref{fig:templates}(\subref{fig:t2}), has vertex set $V = \{r_2, a\}$, edge set $E_2^+ = \emptyset$, and non-edge set $E_2^- = \{(r_2, a)\}$. When $\agg_{T_1}$ and $\agg_{T_2}$  project to the feature vector of $a$, and the outer aggregation functions for $T_1$ and $T_2$ match the two aggregation functions used in AC+ layers, the update rules of node feature vectors for $\T$-GNNs and AC+-GNNs coincide.

\begin{figure}[t]
	\centering 
	\begin{subfigure}[b]{0.48\columnwidth}
		\centering
		\begin{tikzpicture}[scale=0.3] 
			\draw[->] (0,0) -- (1.4,0);
			\draw[fill] (0,0) circle (3pt) node[left] {$r_1$};
			\draw[fill] (1.5,0) circle (3pt) node[right] {$a$};
		\end{tikzpicture}
		\caption{$T_1$}
		\label{fig:t1}
	\end{subfigure}
	\begin{subfigure}[b]{0.48\columnwidth}
		\centering
		\begin{tikzpicture}[scale=0.3] 
			\draw[->,dashed] (0,0) -- (1.4,0);
			\draw[fill] (0,0) circle (3pt) node[left] {$r_2$};
			\draw[fill] (1.5,0) circle (3pt) node[right] {$a$};
		\end{tikzpicture}
		\caption{$T_2$}
		\label{fig:t2}
	\end{subfigure}
	\caption{Templates $T_1$ and $T_2$ used to capture AC-GNNs and AC+-GNNs. The solid arrow represents an element of $E^+$, while the dashed arrow represents an element of $E^-$.}
	\label{fig:templates}
\end{figure}

Another line of work that can be directly interpreted as $\T$-GNNs are the ones that enrich the standard GNNs with local graph structural information. For instance, \cite{chen2025expressive} proposed $k$-hop subgraph GNNs, where every node is updated by incorporating information extracted from subgraphs induced by the set of nodes reachable within $k$ steps from the node. 
Specifically, the k-hop neighbourhood of a node $v$ can be defined as $N_k(v) \dfn \{u \mid d(v, u) \le k\}$ where $d(v, u)$ is the shortest path distance between $v$ and $u$. We write $G_v^k$ for the k-hop subgraph structure rooted at $v$. Instead of aggregating over immediate neighbourhood as in standard GNNs, $k$-hop subgraph GNNs aggregate over $(G_{v}^k, \multiset{\lambda(u)^{(l-1)}}_{u \in N_k(v)})$ when updating the node feature vectors.
$k$-hop subgraph GNNs can be interpreted as $\T$-GNNs where $\T$ is the set of all rooted graphs of radius $k$. The radius of a template $T$ is defined by $rd(T) \coloneqq \max\limits_{v \in V} d(r, v)$,  where $d(r, v)$ is the shortest $E^+$-path distance from $r$ to $v$. If there is no $E^+$-path from $r$ to $v$, we set $d(r, v) := \infty$. 
We provide a concrete example in the next section.

\section{Weisfeiler-Leman Test and Bisimulation}
Next, we define the new notions of template WL-algorithm and graded template bisimulation, and show that the notions coincide and bound the expressive power of template GNNs.
We also discuss how existing WL-variants can be interpreted in our framework.

Fix a countable set of colours C. A node colouring of a labelled graph $G = (V, E, \lambda)$ is a function $\col\colon V \to C$.
Intuitively, template WL algorithm generalises the standard 1-WL algorithm by replacing, in the colour refinement rounds, the multisets of colours of neighbours of a given node by multisets of coloured graphs obtained via template embeddings (or tuples of multisets, when several templates are used).
\begin{definition}[$\T$-WL algorithm]\label{def:TWL}
An $L$-round $\T$-WL algorithm takes as an input a labelled graph $G = (V, E, \lambda)$, a finite set of templates $\T=\{T_1,\dots,T_n\}$, and $L \in \N$. Initial node colours are given by node labels of the graph, $\col^0(v) \dfn \hash(\lambda(v))$ for every $v \in V$. For $1 \le l \le L$, node colours are repeatedly refined at each round by 
	\begin{multline*}
		\col^l(v) \dfn \hash\big(\col^{l-1}(v),\\
		\multiset{(T_1, \col^{l-1}_f) \mid f \in \emb(T_1, (G,v))},\\
		\vdots 
		\\
		\multiset{(T_n, \col^{l-1}_f) \mid f \in \emb(T_n, (G,v))}\big),
	\end{multline*}
where $\col^{l-1}_f(u):= \col^{l-1}(f(u))$, for each $u \in T_i$. The procedure terminates after $L$ rounds, and $\col^L(v)$ is the final node colour. Here $\mathrm{HASH}$ is assumed to be an injective function with co-domain $C$.
\end{definition}

When the $\T$-WL algorithm is run on a pair of graphs, it is run with the same injective $\hash$ function.

The standard 1-WL algorithm updates a node's colour based on its own label and the multiset of its neighbours' colours. $\T$-WL with a single edge template, as shown in Figure \ref{fig:templates}(\subref{fig:t1}), is equivalent to 1-WL. The update rule $\col^l(v) \dfn \hash\big(\col^{l-1}(v), \multiset{(T, \col^{l-1}_f) \mid f \in \emb(T, (G,v))}$ is informationally equivalent to the standard 1-WL update rule: $\col^l(v) \dfn \hash\big(\col^{l-1}(v), \multiset{ (\col^{l-1}(u)) \mid u \in N(v) } \big)$.

To increase the expressive power of standard GNNs (from 1-WL), \cite{chen2025expressive} proposed $k$-hop subgraph GNNs and the corresponding $k$-hop subgraph WL algorithm. As discussed in the previous section, $k$-hop subgraph GNNs can be interpreted as $\T$-GNNs, when $\T$ is the set of all rooted graphs of radius $k$.
Using the same set of templates, $\T$-WL algorithm corresponds to the $k$-hop subgraph WL algorithm of \citeauthor{chen2025expressive}.
For illustrative purposes, the following example shows how 2-hop subgraph WL and $\T$-WL achieve the same result, on a pair of non-isomorphic graphs that cannot be distinguished by standard 1-WL.

\begin{example}
Consider a triangle template $T_{\vartriangle}$ with vertex set $V = \{r, a, b\}$, edge set $E^+ = \{(r, a), (a, b), (b, r)\}$, and non-edge set $E^- = \emptyset$, and a path template $T_p$ with vertex set $V = \{r, a, b\}$, edge set $E^+ = \{(r, a), (a, b)\}$, and non-edge set $E^- = \{(b, r)\}$. Figure \ref{t-wl} shows how $2$-hop subgraph WL and $\T$-WL with $\T = \{T_{\vartriangle}, T_p\}$ distinguish a pair of non-isomorphic graphs, respectively.
\end{example}

	\begin{figure}[t]
		\centering
        \scalebox{0.85}{
		\begin{tikzpicture}[
			node distance=1cm,
			vnode/.style={circle, draw=black, thin, minimum size=4.5mm, inner sep=0pt, font=\bfseries\scriptsize},
			subgraph_node/.style={circle, draw=black, thin, minimum size=3mm, inner sep=0pt, font=\tiny},
			template_node/.style={circle, fill=black, draw=none, minimum size=3pt, inner sep=0pt},
			label_text/.style={font=\sffamily\scriptsize, align=center},
			header/.style={font=\bfseries\small},
			diredge/.style={draw, semithick, -{Stealth[length=1.2mm, width=0.9mm]}},
			mini_edge/.style={draw, thin, -{Stealth[length=0.8mm, width=0.6mm]}},
			extract_arrow/.style={->, >=stealth, dashed, color=gray, thick}
			]
			
			\node[header] at (4, 5.2) {2-hop subgraph WL};

			\node[label_text] at (2, 4.6) {\textbf{Induced subgraph}};
			\begin{scope}[shift={(2, 3.8)}]
				\node[subgraph_node, fill=red!10] (e1_1) at (90:0.4) {};
				\node[subgraph_node, fill=red!10] (e1_2) at (210:0.4) {};
				\node[subgraph_node, fill=red!10] (e1_3) at (330:0.4) {};
				\draw[mini_edge] (e1_1) -- (e1_2);
				\draw[mini_edge] (e1_2) -- (e1_3);
				\draw[mini_edge] (e1_3) -- (e1_1); 
				\node[font=\tiny, below=0.1cm] at (0, -0.2) {};
			\end{scope}

			\node[label_text] at (6, 4.6) {\textbf{Induced subgraph}};
			\begin{scope}[shift={(6, 3.8)}]
				\node[subgraph_node, fill=orange!10] (e2_1) at (90:0.4) {}; 
				\node[subgraph_node, fill=orange!10] (e2_2) at (330:0.4) {};
				\node[subgraph_node, fill=orange!10] (e2_3) at (210:0.4) {};
				\draw[mini_edge] (e2_1) -- (e2_2);
				\draw[mini_edge] (e2_2) -- (e2_3);
				\node[font=\tiny, below=0.1cm] at (0, -0.2) {};
			\end{scope}

			\node[label_text] at (0.2, 2.2) {$G_1$};

			\begin{scope}[shift={(1.1, 2.0)}]
				\node[vnode, fill=red!30] (t1_1) at (90:0.55) {$v$};
				\node[vnode, fill=red!30] (t1_2) at (210:0.55) {};
				\node[vnode, fill=red!30] (t1_3) at (330:0.55) {};
				\draw[diredge] (t1_1) -- (t1_2);
				\draw[diredge] (t1_2) -- (t1_3);
				\draw[diredge] (t1_3) -- (t1_1);
			\end{scope}

			\begin{scope}[shift={(2.9, 2.0)}]
				\node[vnode, fill=red!30] (t2_1) at (90:0.55) {};
				\node[vnode, fill=red!30] (t2_2) at (210:0.55) {};
				\node[vnode, fill=red!30] (t2_3) at (330:0.55) {};
				\draw[diredge] (t2_1) -- (t2_2);
				\draw[diredge] (t2_2) -- (t2_3);
				\draw[diredge] (t2_3) -- (t2_1);
			\end{scope}

			\node[label_text] at (4.5, 2.2) {$G_2$};

			\begin{scope}[shift={(6.0, 2.0)}]
				\foreach \i in {1,...,6} {
					\node[vnode, fill=orange!30] (h\i) at ({90-(\i-1)*60}:0.85) {}; 
				}
				\node[font=\bfseries\scriptsize] at (h1) {$v$}; 
				\draw[diredge] (h1) -- (h2);
				\draw[diredge] (h2) -- (h3);
				\draw[diredge] (h3) -- (h4);
				\draw[diredge] (h4) -- (h5);
				\draw[diredge] (h5) -- (h6);
				\draw[diredge] (h6) -- (h1);
			\end{scope}
			
			\draw[extract_arrow] (t1_1) to[out=110, in=180] (e1_1);
			\draw[extract_arrow] (h1) to[out=165, in=180] (e2_1);

			\draw[dashed, thin, gray!80] (0, 0.8) -- (8, 0.8);

			\node[header] at (4, 0.3) {$\mathcal{T}$-WL};

			\node[label_text] at (2, -0.4) {\textbf{$T_{\vartriangle}$}};
			\begin{scope}[shift={(2, -1.2)}]
				\node[template_node] (tmp1_1) at (90:0.4) {};
				\node[template_node] (tmp1_2) at (210:0.4) {};
				\node[template_node] (tmp1_3) at (330:0.4) {};
				\draw[mini_edge] (tmp1_1) -- (tmp1_2);
				\draw[mini_edge] (tmp1_2) -- (tmp1_3);
				\draw[mini_edge] (tmp1_3) -- (tmp1_1);
			\end{scope}

			\node[label_text] at (6, -0.4) {\textbf{$T_p$}};
			\begin{scope}[shift={(6, -1.2)}]
				\node[template_node] (tmp2_1) at (90:0.4) {}; 
				\node[template_node] (tmp2_2) at (330:0.4) {};
				\node[template_node] (tmp2_3) at (210:0.4) {};
				\draw[mini_edge] (tmp2_1) -- (tmp2_2);
				\draw[mini_edge] (tmp2_2) -- (tmp2_3);
                \draw[mini_edge, dashed] (tmp2_3) -- (tmp2_1);
			\end{scope}

			\node[label_text] at (0.2, -2.8) {$G_1$};

			\begin{scope}[shift={(1.1, -3.0)}]
				\node[vnode, fill=red!30] (bt1_1) at (90:0.55) {$v$};
				\node[vnode, fill=red!30] (bt1_2) at (210:0.55) {};
				\node[vnode, fill=red!30] (bt1_3) at (330:0.55) {};
				\draw[diredge] (bt1_1) -- (bt1_2);
				\draw[diredge] (bt1_2) -- (bt1_3);
				\draw[diredge] (bt1_3) -- (bt1_1);
			\end{scope}

			\begin{scope}[shift={(2.9, -3.0)}]
				\node[vnode, fill=red!30] (bt2_1) at (90:0.55) {};
				\node[vnode, fill=red!30] (bt2_2) at (210:0.55) {};
				\node[vnode, fill=red!30] (bt2_3) at (330:0.55) {};
				\draw[diredge] (bt2_1) -- (bt2_2);
				\draw[diredge] (bt2_2) -- (bt2_3);
				\draw[diredge] (bt2_3) -- (bt2_1);
			\end{scope}

			\node[label_text] at (4.5, -2.8) {$G_2$};

			\begin{scope}[shift={(6.0, -3.0)}]
				\foreach \i in {1,...,6} {
					\node[vnode, fill=orange!30] (bh\i) at ({90-(\i-1)*60}:0.85) {};
				}
				\node[font=\bfseries\scriptsize] at (bh1) {$v$};
				\draw[diredge] (bh1) -- (bh2);
				\draw[diredge] (bh2) -- (bh3);
				\draw[diredge] (bh3) -- (bh4);
				\draw[diredge] (bh4) -- (bh5);
				\draw[diredge] (bh5) -- (bh6);
				\draw[diredge] (bh6) -- (bh1);
			\end{scope}
			
			\draw[extract_arrow] (bt1_1) to[out=110, in=180] (tmp1_1);
			\draw[extract_arrow] (bh1) to[out=165, in=180] (tmp2_1);
		\end{tikzpicture}
        }
		\caption{Top: $2$-hop subgraph WL extracts subgraphs rooted at node $v$. Bottom: $\mathcal{T}$-WL arrives at the same colouring by matching node $v$ against $\T = \{T_{\vartriangle}, T_p\}$.}
		\label{t-wl}
	\end{figure}
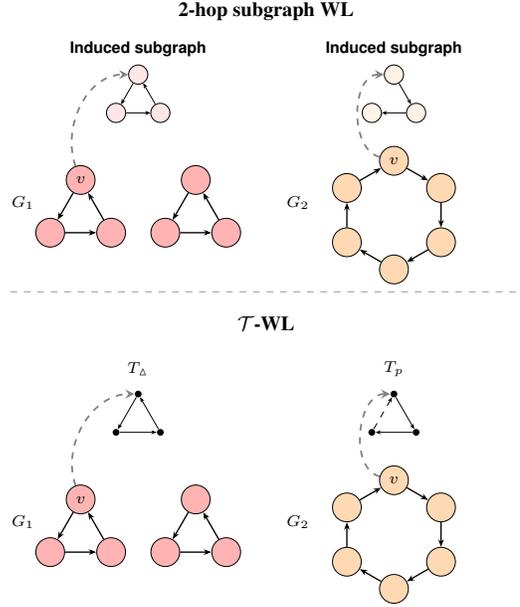

Next, we introduce the notion of graded template bisimulation and connect it to the $\T$-WL algorithm.
\begin{definition}[Graded $\T$-bisimulation] 
Let $G = (V, E, \lambda)$ and $G' = (V', E', \lambda')$ be labelled graphs and $\T$ a set of templates. A relation $Z_0 \subseteq V \times V'$ is a graded $0$-$\T$-bisimulation if for every $(v, v') \in Z_0$, $\lambda(v) = \lambda'(v')$. For $l \ge 1$, a relation $Z_l \subseteq V \times V'$ is a graded $l$-$\T$-bisimulation if there exists a graded ($l$-$1$)-$\T$-bisimulation $Z_{l-1}$  s.t.  for every $(v, v') \in Z_l$:
\begin{enumerate}
\item $(v, v') \in Z_{l-1}$,
\item for every $T \in \T$, and every $k\geq 1$: \label{item:Tbis}
\begin{itemize}
\item for pairwise distinct $f_1,\dots f_k\in \emb(T, (G, v))$ there exists pairwise distinct $f'_1,\dots f'_k\in \emb(T, (G', v'))$ s.t. for all $1\leq i\leq k$, $u \in T$, $(f_i(u), f_i'(u)) \in Z_{l-1}$,
\item for pairwise distinct $f'_1,\dots f'_k\in \emb(T, (G', v'))$ there exists pairwise distinct $f_1,\dots f_k\in \emb(T, (G, v))$  s.t. for all $1\leq i\leq k$, $u \in T$, $(f_i(u), f_i'(u)) \in Z_{l-1}$.
\end{itemize}

\end{enumerate} 
We say that $(v,v')$ are \emph{graded $l$-$\T$-bisimilar}---and write $(G',v') \sim^{l}_{\T} (G,v)$---if there is a graded $l$-$\T$-bisimulation $Z_l$ such that $(v,v')\in Z_l$.
For a counting bound $c \ge 1$, we define $l$-$c$-$\T$-bisimulation by 
restricting $k\leq c$ in item \ref{item:Tbis}, and write $(G',v') \sim^{l,c}_{\T} (G,v)$ when $v$ and $v'$ are $l$-$c$-$\T$-bisimilar.

We obtain the notion of \emph{graded $\T$-bisimulation} (\emph{graded $\T$-bisimilarity}, resp.) from the above definitions in the obvious way by replacing all occurrences of $Z_0$, $Z_l$, and $Z_{l-1}$ by $Z$. For the cases of graded $\T$-bisimulation and graded $l$-$\T$-bisimulation, we may the use the following alternative formulation of item \ref{item:Tbis}:
\begin{enumerate}
\item[2'.] for every template $T \in \T$, there exists a bijection $g: \emb(T, (G, v)) \to \emb(T, (G', v'))$ such that for all $f \in \emb(T, (G, v))$, for all $u \in T$, $(f(u), g(f)(u)) \in Z_{l-1}$.
\end{enumerate}
\end{definition}

$\T$-WL algorithm and graded $\T$-bisimulation are different sides of the same coin, as the following proposition shows.
\begin{proposition}\label{prop:WL-bis}
Let $G = (V, E, \lambda)$ and $G' = (V', E', \lambda')$ be labelled graphs, $\T$ a finite set of templates, and $l \in \N$.
Then $(G',v') \sim^{l}_{\T} (G,v)$ if and only if $\col^l(v) = \col^l(v')$.
\end{proposition}
\begin{proof}
The proof is by induction on $l$. 
For the base case, by definition of initial colouring, $\col^0(v) = \hash(\lambda(v))$ and $\col^0(v') = \hash(\lambda'(v'))$. Since the hash function is injective, $\col^0(v) = \col^0(v')$ if and only if $\lambda(v) = \lambda'(v')$. By definition of graded $0$-$\T$-bisimulation, $(v, v') \in Z_0$ if and only if $\lambda(v) = \lambda'(v')$.
For the inductive step, assume the proposition holds for $l-1$.

($\Leftarrow$):
Since the hash function is injective, $\col^l(v) = \col^l(v')$ holds if and only if the arguments to the hash function are identical, as follows
\begin{enumerate}
	\item $\col^{l-1}(v) = \col^{l-1}(v')$,
	\item for every $T \in \T$, $\multiset{(T, \col^{l-1}_f) \mid f \in \emb(T, (G, v))} = \multiset{(T, \col^{l-1}_{f'}) \mid f' \in \emb(T, (G',v'))}$
\end{enumerate}
By induction hypothesis, 1.  implies $(v, v') \in Z_{l-1}$. The multisets are equal in 2. implies for every $T \in \T$, there exists a bijection $g: \emb(T, (G, v)) \to \emb(T, (G', v'))$ such that for all $f \in \emb(T, (G, v))$, $(T, \col^{l-1}_f) = (T, \col^{l-1}_{g(f)})$, which implies 
$\col^{l-1}_f = \col^{l-1}_{g(f)}$, which means for all $u \in T$, $\col^{l-1}(f(u)) = \col^{l-1}(g(f)(u))$ by definition. By induction hypothesis, it follows that for all $u \in T$, $(f(u), g(f)(u)) \in Z_{l-1}$. Since $\lambda(v) = \lambda'(v')$ is implied by initial colouring, $(v, v') \in Z_{l-1}$ is implied by 1., and the existence of bijection $g$ for all $T \in \T$ is implied by 2., the definition of graded $l$-$\T$-bisimulation is satisfied. Therefore, $(v, v') \in Z_l$.

($\Rightarrow$): Assume $(v, v') \in Z_l$. Then $(v, v') \in Z_{l-1}$ by definition. By induction hypothesis, $(v, v') \in Z_{l-1}$ implies $\col^{l-1}(v) = \col^{l-1}(v')$, matching the first argument of the hash function. Since $(v, v') \in Z_l$, for every $T \in \T$, there exists a bijection $g: \emb(T, (G, v)) \to \emb(T, (G', v'))$ such that for all $f \in \emb$, for all $u \in T$, $(f(u), g(f)(u)) \in Z_{l-1}$ by definition. Applying induction hypothesis to all pairs $(f(u), g(f)(u)) \in Z_{l-1}$, we have $\col^{l-1}(f(u)) = \col^{l-1}(g(f)(u))$, which means $\col^{l-1}_f = \col^{l-1}_{g(f)}$. Since for every $T \in \T$ a bijection $g$ exists mapping $f$ to $g(f)$, $\multiset{(T, \col^{l-1}_f) \mid f \in \emb(T, (G, v))} = \multiset{(T, \col^{l-1}_{f'}) \mid f' \in \emb(T, (G', v'))}$, matching the other arguments of the hash function. Therefore, $\col^l(v) = \col^l(v')$.
\end{proof}
It is not difficult too see that the above proof gives an analogous connection between the bounded variants of $\T$-WL-algorithm and bounded $\T$-bisimulation. For $c\geq 1$, a $c$-bounded $\T$-WL algorithm is obtained from Definition \ref{def:TWL} by stipulating that $\hash$ cannot differentiate between multiplicities $c$ from larger multiplicities (i.e., $\hash$ is a $c$-bounded function). Hence, Proposition \ref{prop:WL-bis} can be reformulated to state that $(v, v')$ are $l$-$c$-$\T$-bisimilar if and only if $\col^l(v) = \col^l(v')$ using a $c$-bounded $\hash$-function. 

\begin{proposition}\label{prop:GNNbisimulation}
	Let $G$ and $G'$ be labelled graphs, $\T$ be a finite set of templates, $\mathcal{N}$ a $c$-bounded $L$-layer $\T$-GNN, and $l\leq L$. If $(v, v') \in V \times V'$ are $l$-$c$-$\T$-bisimilar, then they have the same feature vectors at the $l$-th layer of $\mathcal{N}$.
    The result remains true also for non-bounded $\T$-GNN, if $l$-$c$-$\T$-bisimilarity is replaced with graded $l$-$\T$-bisimilarity.
\end{proposition}

\begin{proof}
The proof is by induction on $l$. For the base case,	by definition of $0$-$c$-$\T$-bisimulation, $(v, v') \in Z_0$ if and only if $\lambda(v) = \lambda(v')$. $\T$-GNN $\GN$ initializes node feature vectors based on the initial node labels $\lambda(v)$ and $\lambda(v')$ of the graphs. For the inductive step, assume the statement holds for $l-1$.

Consider the update rule for a $c$-bounded $\T$-GNN classifier. Independent of choices of $\agg$, $\agg_T$, and $\comb$ functions, $\lambda^l(v) = \lambda^l(v')$ holds if 
\begin{enumerate}
	\item $\lambda^{l-1}(v) = \lambda^{l-1}(v')$,
	\item for every $T \in \{T_1^l, \dots, T_n^l\}$, 
    \begin{multline*}
    \multiset{(T, \lambda^{l-1}_f) \mid f \in emb(T, (G,v))}\\
    =_c \multiset{(T, \lambda^{l-1}_{f'}) \mid f' \in emb(T, (G',v'))}.
    \end{multline*}
\end{enumerate}

Assume $(v, v') \in Z_l$. Then $(v, v') \in Z_{l-1}$ by definition. By induction hypothesis $(v, v') \in Z_{l-1}$ implies $\lambda^{l-1}(v) = \lambda^{l-1}(v')$.
By definition of $l$-$c$-$\T$-bisimulation, we have for every $T \in \T$, $k \le c$, for pairwise distinct $f_1,\dots f_k\in \emb(T, (G, v))$ there exists pairwise distinct $f'_1,\dots f'_k\in \emb(T, (G', v'))$ s.t. for all $1\leq i\leq k$, $u \in T$, $(f_i(u), f_i'(u)) \in Z_{l-1}$. Applying induction hypothesis to all pairs $(f_i(u), f'_i(u)) \in Z_{l-1}$, we have for all $1\leq i\leq k$, $u \in T$, $\lambda^{l-1}(f_i(u)) = \lambda^{l-1}(f'_i(u))$, which means $\lambda^{l-1}_f = \lambda^{l-1}_{f'}$. The same arguments apply to the back condition symmetrically. The back and forth conditions ensure that the number of embeddings mapping to any tuple of equivalent classes in $Z_{l-1}$ is preserved between $(G, v)$ and $(G', v')$ for $k \le c$. 
Hence, for every $T \in \{T_1^l, \dots, T_n^l\}$, $\multiset{(T, \lambda^{l-1}_f) \mid f \in \emb(T, (G, v))} =_c \multiset{(T, \lambda^{l-1}_{f'}) \mid f' \in \emb(T, (G', v'))}$. Therefore, $\lambda^l(v) = \lambda^l(v')$.

The proof for non-bounded case follows in verbatim, once all references to the counting bound $c$ are removed.
\end{proof}

\section{Graded Modal Logic and Template GNNs}
Next, we define our new logic---graded template modal logic $\GML(\T)$ and prove the main technical result of the paper: for any finite set of templates $\T$, $\GML(\T)$ captures the uniform expressivity of bounded counting $\T$-GNNs.

\subsection{Graded Template Modal Logic}
We define a logic that has a modality $\gtmod{T}{j}(\varphi_1,\dots,\varphi_n)$ for each template $T$ of cardinality $n+1$ and positive $j\in\N$. From now on, we stipulate that the domain of a template of cardinality $n+1$ is the set $[n+1]\dfn \{0,1,\dots,n\}$ and that $0$ is the root of the template.

\vspace{2mm}
\noindent\textbf{GML($\T$).}
Let $\T$ be a finite set of templates. The syntax of
GML($\T$) is given by the following grammar:
\[
\varphi := p \mid \neg\varphi \mid  \varphi \wedge \varphi \mid \langle T \rangle^{\ge j} (\varphi_1, \varphi_2, \ldots, \varphi_{n_T}),
\]
where $j\geq 1$ is a natural number, $T = (V, E^+, E^-, r) \in \T$ is a template, and $n_T= \lvert V \rvert -1$.

The semantics extends that of GML with the following clause for formulae of the form $\langle T \rangle^{\ge j} (\varphi_1, \varphi_2, \ldots, \varphi_{n_T})$:
$(G, v) \models \langle T \rangle^{\ge j} (\varphi_1, \varphi_2, \ldots, \varphi_{n_T})$ if and only if
\begin{multline*}
\lvert \{ f \in \emb(T, (G,v)) \\
\mid (G,f(i)) \models \varphi_i, \text{ for }1 \leq i \leq n_T \}\rvert \geq j.
\end{multline*}

The modal depth $\md(\varphi)$ of a $\GML(\T)$-formula is defined as usual, with the following additional case:
\[
\md(\langle T \rangle^{\ge j} (\varphi_1, \varphi_2, \ldots, \varphi_n)) \dfn 1 + \max\limits_{1 \leq i \leq n}(\md(\varphi_i)).
\]
The depth of the syntactic tree $\sd(\varphi)$ of $\varphi$ is defined similarly as $\md(\varphi)$, except that all logical connectives add 1 to the depth.
We define the \emph{counting bound} $\cb(\varphi)$ of a formula $\varphi\in \GML(\T)$ to be the smallest natural number $c\in\N$ such that if $\langle T \rangle^{\ge j}$ occurs in $\varphi$, then $j\leq c$.

\begin{proposition} \label{inv}
Every GML($\T$)-formula of modal depth at most $l$ and counting bound at most $c$ is invariant under $l$-$c$-$\T$-bisimulation.
\end{proposition}
\begin{proof}
	The proof is by induction on the structure of $\varphi$. For the base case, let $\varphi=p$ be a proposition symbol. By definition, if $(G,v)$ and $(G',v')$ are $l$-$c$-$\T$-bisimilar, then $\lambda(v)=\lambda(v')$. Hence
	\(
		(G, v) \models p 
		\Leftrightarrow p \in \lambda(v) 
		  \Leftrightarrow p \in \lambda'(v')  
		  \Leftrightarrow (G', v') \models p.
	\)
	
	For the inductive step, let $\varphi$ be a formula with $\md(\varphi) \le l$ and $\cb \leq c$. The Boolean cases follow immediately from the induction hypothesis. Consider the case where $\varphi = \langle T \rangle^{\ge j} (\varphi_1, \ldots, \varphi_n)$. Since $\md(\varphi) \le l$, we have $\md(\varphi_i) \le l-1$, for $1 \leq i \leq n$.  By definition of counting bound, $j \leq c$. 
	
	Let  $S$ be the set
	\[\{ f \in \text{emb}(T, (G,v)) \mid (G,f(i)) \models \varphi_i, \text{ for } 1 \leq i \leq n \}\]
	and $S'$ be the corresponding set
	\[
    \{ f' \in \text{emb}(T, (G',v')) \mid (G',f'(i)) \models \varphi_i, \text{ for } 1 \leq i \leq n \}.
    \]
	By the forth condition in the definition of $l$-$c$-$\T$-bisimulation, for pairwise distinct $f_1,\dots f_j \in \emb(T, (G, v))$ there exists pairwise distinct $f'_1,\dots f'_j \in \emb(T, (G', v'))$ s.t. for all $1\leq m \leq j$, $u \in T$, $(f_m(u), f_m'(u)) \in Z_{l-1}$. Symmetrically, by the back condition in the definition of $l$-$c$-$\T$-bisimulation, we have for pairwise distinct $f'_1,\dots f'_j\in \emb(T, (G', v'))$ there exists pairwise distinct $f_1,\dots f_j \in \emb(T, (G, v))$  s.t. for all $1\leq m \leq j$, $u \in T$, $(f_m(u), f_m'(u)) \in Z_{l-1}$. Since $(f_m(u), f_m'(u)) \in Z_{l-1}$ and $\md(\varphi_i) \le l-1$, by induction hypothesis, we have $(G, f(i)) \models \varphi_i \Leftrightarrow (G', f'(i)) \models \varphi_i$, for $1 \leq i \leq n$. Therefore, 
	\(
		(G, v) \models \varphi
		 \,\Leftrightarrow\, |S|  \ge j
		 \,\Leftrightarrow\, |S'|  \ge j
		 \,\Leftrightarrow\, (G', v') \models \varphi.
	\)
\end{proof}
Next we define two variants of characteristic formulae $\chi^{l}_{(G,v)}$ and  $\chi^{l,c}_{(G,v)}$ for $\GML(\T)$, where $l,c\in\N$ and $(G,v)$ is a finite pointed graph. The goal of these formulae is to satisfy the following two properties 
\begin{align*}
(G',v')\models \chi^{l}_{(G,v)} \,&\Leftrightarrow\, (G',v') \sim^{l}_{\T} (G,v)\\
(G',v')\models \chi^{l,c}_{(G,v)} \,&\Leftrightarrow\, (G',v') \sim^{l,c}_{\T} (G,v).
\end{align*}
Since we are interested in connections to neural networks, we restrict the consideration to finite graphs. In the  case of infinite graphs, $\chi^{l}_{(G,v)}$ would not always be a finite formula.

For a template $T$ of cardinality $n+1$, an $n$-tuple of $\GML$-formulae $\vec{\varphi}$, and a pointed graph $(G,v)$, we define $S^{(G,v)}_{T,\vec{\varphi}}$ to be the following set of template embeddings:
    \begin{multline*}
 \{ f \in \emb(T, (G,v)) \mid (G, f(i)) \models \varphi_i \text{ for } 1 \leq i \leq n \}.
	\end{multline*}

\begin{definition}[Characteristic formulae for $\GML$]
Let $(G,v)$ be a pointed labelled graph.
We define the \emph{$l$-characteristic formula} $\chi^{l}_{(G,v)}$ of $(G,v)$ by induction on $l \in \N$ as follows:
\[\chi_{(G, v)}^{0} \dfn \bigwedge \{p  \mid p \in \lambda(v)\} \wedge \bigwedge \{\neg p \mid p \notin \lambda(v)\}. \]
For $l \ge 1$, the characteristic formula $\chi_{(G, v)}^{l}$ is defined as 
	\begin{multline*}
		\chi_{(G, v)}^{l} := \chi_{(G, v)}^{l-1} \\
		\wedge \bigwedge_{T \in \mathcal{T}} \Big( 
        \hspace{-2mm}\bigwedge_{f \in \emb(T, (G,v))} \hspace{-7mm}\langle T \rangle^{\ge k} (\varphi^f_1, \ldots, \varphi^f_n)  \,\,\,\text{where }k=\vert S^{(G,v)}_{T,\vec{\varphi}}\rvert \\
		\wedge \neg\langle T \rangle^{\ge \lvert S^{(G,v)}_{T,\vec{\top}}\rvert + 1} (\vec{\top})\Big),
	\end{multline*}
    where $\vec{\varphi}= (\varphi^f_1, \ldots, \varphi^f_n) = (\chi_{(G, f(1))}^{l-1}, \ldots, \chi_{(G, f(n))}^{l-1})$ and $\vec{\top}=(\top, \dots,\top)$.
\end{definition}

Next, we define the bounded variants of $l$-characteristic formulae. We first require the following proposition, whose proof is standard.
We say that an equivalence relation $\sim$ has a \emph{finite index}, if it has finitely many equivalence classes. We write $[\sim]$ for the set of all $\sim$ equivalence classes.
\begin{proposition}\label{prop:finiteindex}
For every $l,c \in\N$ and a finite set of templates $\T$, the relation $\sim^{l,c}_\T$ has a finite index.
\end{proposition}
For the next definition, we need a representative for each equivalence class $C\in [\sim^{l,c}_\T]$; we may pick, e.g., the first in the alphabetical ordering obtained via some encoding. We write $(G,v)\in [\sim^{l,c}_\T]$ when ranging over such representatives.

\begin{definition}[Bounded characteristic formulae for $\GML$]
Let $c\geq 1$ be a natural number, $\T$ a finite set of templates, and $(G,v)$ a pointed labelled graph.
We define the \emph{l-c-characteristic formula} $\chi^{l,c}_{(G,v)}$ of $(G,v)$ by induction on $l \in \N$ as follows:
\[\chi_{(G, v)}^{0,c} \dfn \bigwedge \{p  \mid p \in \lambda(v)\} \wedge \bigwedge \{\neg p \mid p \notin \lambda(v)\}. \]
For $l \ge 1$, the characteristic formula $\chi_{(G, v)}^{l,c}$ is defined as 
	\begin{multline*}
		\chi_{(G, v)}^{l,c} := \chi_{(G, v)}^{l-1,c} \\
		\wedge \bigwedge_{T \in \mathcal{T}} \Big( 
        \hspace{-3mm}\bigwedge_{f \in \emb(T, (G,v))} \hspace{-8mm}\langle T \rangle^{\ge k} (\vec{\varphi})  \quad\text{where } k=\min\{c, \lvert S^{(G,v)}_{T,\vec{\varphi}}\rvert\} \\
		\wedge \hspace{-3mm}\bigwedge_{\substack{\{(G_i,v_i) \in [\sim^{l-1,c}_\T]\}_{1\leq i\leq n}\\ \lvert S^{(G,v)}_{T,\vec{\psi}}\rvert+1\leq c}} \hspace{-8mm} \neg\langle T \rangle^{\ge \lvert S^{(G,v)}_{T,\vec{\psi}}\rvert+1} (\psi_1, \ldots, \psi_n) \Big),
	\end{multline*}
    where $\vec{\varphi}= (\varphi^f_1, \ldots, \varphi^f_n) = (\chi_{(G, f(1))}^{l-1,c}, \ldots, \chi_{(G, f(n))}^{l-1,c})$,
    $\vec{\psi}= (\psi_1, \ldots, \psi_n) = (\chi_{(G_1, v_1)}^{l-1,c}, \ldots, \chi_{(G_n, v_n)}^{l-1,c})$.
\end{definition}
The following proposition follows directly from the construction of the characteristic formulae.
\begin{proposition} \label{cf}
Every pointed graph $(G,v)$ satisfies its own $l$-$c$-$\T$-characteristic formula $\chi^{l,c}_{(G,v)}$, for every $l,c\in\N$.
\end{proposition}

We are now ready to prove the desired property that a characteristic formula of a pointed graph characterises the corresponding bisimulation equivalence class.
\begin{proposition}\label{prop:characteristicformula}
Let $c\geq 1$ and $l\in\N$ be natural numbers and $\T$ a finite set of templates. 
Then
\[
(G',v')\models \chi^{l,c}_{(G,v)} \Leftrightarrow (G',v') \sim^{l,c}_{\T} (G,v).
\]
Hence, every $\sim^{l,c}_{\T}$ equivalence class is definable by a GML($\T$) formula of modal depth $l$ and counting bound $c$.
\end{proposition}

\begin{proof}	
$(\Leftarrow)$:
Assume $(G',v') \sim^{l,c}_\T (G,v)$. By Prop.~\ref{cf}, $(G,v) \models \chi^{l,c}_{(G,v)}$. By Prop.~\ref{inv}, therefore, $(G',v') \models \chi^{l,c}_{(G,v)}$.

$(\Rightarrow)$:
The proof proceeds by induction on $l$. For the base case, by definition of $\chi^{0,c}_{(G,v)}$, $(G',v') \models \chi^{0,c}_{(G,v)}$ iff $\lambda(v') = \lambda(v)$, matching the definition of $(G',v') \sim^{0,c}_\T (G,v)$. For the inductive step, assume the proposition holds for $l-1$. 

Assume $(G',v') \models \chi^{l,c}_{(G,v)}$. By definition of $\chi^{l,c}_{(G,v)}$, $(G',v') \models \chi^{l-1,c}_{(G,v)}$. By induction hypothesis, $(G',v') \sim^{l-1,c}_\T (G,v)$, matching the first condition in the definition of $\sim^{l,c}_\T$.
For any $T \in \T$, let $\vec{C} = (C_1, \ldots, C_n)$ be any tuple of representatives of the equivalence classes of $\sim^{l-1,c}_\T$, and $\vec{\psi}_{\vec{C}}$ be the tuple of corresponding characteristic formulae.
Let $n = |S^{(G,v)}_{T, \vec{\psi}_{\vec{C}}}|$ and $m = |S^{(G',v')}_{T, \vec{\psi}_{\vec{C}}}|$.
In the case of $n \ge c$, the second conjunct in the definition of $\chi^{l,c}_{(G,v)}$ requires $\langle T \rangle^{\ge c} \vec{\psi}$, which implies $m \ge c$.
The third conjunct does not apply as $n + 1 \le c$ is not met. Hence, $\min(c, n) = c = \min(c, m)$.
In the case of $n < c$, the second conjunct requires $\langle T \rangle^{\ge n} \vec{\psi}$, which implies $m \ge n$. As $n + 1 \le c$, the third conjunct applies and requires $\neg \langle T \rangle^{\ge n+1} \vec{\psi}$, which implies $m \le n$. Taken together, $m = n$. Hence, $\min(c, n) = n = m = \min(c, m)$.

Let $f_1, \ldots, f_k \in \emb(T, (G,v))$ be pairwise distinct embeddings, $k \le c$. Group them by the characteristic formulae $\vec{\psi}_{\vec{C}}$ satisfied by the image of template nodes. For any specific group of size $k' \le k \le c$, $k' \le n$. Hence, $k' \le \min(c, n) = \min(c, m)$. This implies $G'$ contains at least $k'$ distinct embeddings where the image of template nodes satisfy $\vec{\psi}_{\vec{C}}$. By induction hypothesis, nodes satisfying the same characteristic formulae are $($l$-$1$)$-$c$-$T$-bisimilar. Symmetric arguments apply to the back condition in the definition of $\sim^{l,c}_\T$. 
Therefore, $(G',v') \sim^{l,c}_\T (G,v)$.
\end{proof}
Now, since there are only finitely many bounded bisimulation classes for any fixed parameters, we obtain that every $l$-$c$-$\T$ invariant class of pointed graphs is definable in $\GML(\T)$.
\begin{proposition}\label{prop:bisGMLdef}
Every $l$-$c$-$\T$ invariant class of pointed graphs is definable by $\GML(\T)$ formula of modal depth $l$ and counting bound $c$.
\end{proposition}
\begin{proof}
 Every $l$-$c$-$\T$ invariant class of pointed graphs is a union of $\sim^{l,c}_\T$ equivalence classes. Every $\sim^{l,c}_\T$ equivalence class is definable by a $l$-$c$-characteristic formula by Proposition \ref{prop:characteristicformula}, and there is a finite number of such equivalence classes by Proposition \ref{prop:finiteindex}. Thus, every $l$-$c$-$\T$ invariant class of pointed graphs is definable by a finite disjunction of $l$-$c$-characteristic formulae, one for each equivalence class.
\end{proof}

\subsection{From GNNs to logic and back}
We are now ready to prove our main technical results connecting the uniform expressivity of $\T$-GNNs with expressivity of $\GML(\T)$.
Combining our results on $\T$-GNN invariance under bounded graded bisimulation and Proposition \ref{prop:bisGMLdef} allows us to establish the following result of the uniform expressivity of $\T$-GNNs.
Recall that a logical classifier $\varphi$ captures a (Boolean) GNN classifier $\mathcal{N}$ if for every graph $G$ and node $v$ in $G$, it holds that $\mathcal{N}(G, v) = 1$ if and only if $(G, v) \models \varphi$.

\begin{theorem}\label{thm:GNNtoLogic}
    Let $\T$ be a finite set of templates and $\GN$ a $c$-bounded $\T$-GNN with $l$ layers. Then there exists a $\varphi\in \GML(\T)$ of modal depth $l$ and counting bound $c$ that captures $\GN$.
\end{theorem}
\begin{proof}
The class of pointed graphs that a $c$-bounded $\T$-GNN $\GN$ with $l$ layers accepts is invariant under $l$-$c$-$\T$ bisimulation by Proposition \ref{prop:GNNbisimulation}. Every such class of pointed graphs is definable by a $\GML(\T)$ formula of modal depth $l$ and counting bound $c$ by Proposition \ref{prop:bisGMLdef}.
\end{proof}

\begin{theorem}\label{thm:LogictoGNN}
Let $\T$ be a finite set of templates, $\varphi\in \GML(\T)$, $c_{max}=\cb(\varphi)$ and $l=sd(\varphi)$. Then there exists a (bounded) $\T$-GNN $\GN$ with $l$ layers that captures $\varphi$. 
\end{theorem}
\begin{proof}

Let $sub(\varphi) = (\varphi_1, \ldots, \varphi_d)$ be the set of all subformulae of $\varphi$, such that if $\varphi_k$ is a subformula of $\varphi_l$, then $k \leq l$.
In particular, $\varphi = \varphi_d$.

We construct a (bounded) $l$ layer $\T$-GNN $\GN$ of dimension $d$. The final classification function is $\cls(\lambda^l(v)_d)$, where $\lambda^l(v)_d$ is the $d$-th component of $\lambda^l(v)$.

Let $J \subseteq \{1, \dots, d\}$ be the set of indices such that for all $k \in J$, $\varphi_k$ is the modal formula of the form $\langle T \rangle^{\ge c} (\varphi_1, \varphi_2, \ldots, \varphi_n)$. Let $m = |J|$ be the number of modal subformulae in $sub(\varphi)$. 
We define a bijection $\iota : \{1, \dots, m\} \to J$ such that for each $j \in \{1, \dots, m\}$, let the corresponding subformula be:
\[ \varphi_{\iota(j)} = \langle T^{(j)} \rangle^{\ge c_j} (\psi_{j,1}, \dots, \psi_{j,n_j}). \]

Layer $0$ initializes all feature vectors to $\mathbb{R}^d$ to represent truth values of propositions, such that if $\varphi_k$ is a proposition, then $\lambda^0(v)_k = 1$ iff $(G,v) \models \varphi_k$, all other entries are set to 0. Layers $1, \ldots, l$ are homogeneous, with activation function being truncated ReLU, defined as 
$\sigma(x) = \min(\max(0,x), 1)$. Aggregation function is the max-$n$-sum where $n$ is the counting bound of $\varphi$. For each layer $l$ and each $j \in \{1, \dots, m\}$, assign the template $T_j^l = T^{(j)}$, corresponding to $\varphi_{\iota(j)} = \langle T^{(j)} \rangle^{\ge c_j} (\psi_{j,1}, \dots, \psi_{j,n_j})$, let $id(j,i)$ be the index of the subformula $\psi_{j,i}$ in $sub(\varphi)$. The template aggregation function is defined as follows:
\begin{equation}\label{eq:aggT}
\small
\hspace{-3mm}
\agg^l_{T^{(j)}}(T^{(j)}, \lambda^{l-1}_f) \dfn \sigma\left( \sum_{i=1}^{n_j} \lambda^{l-1}(f(i))_{id(j,i)} - n_j + 1 \right)
\end{equation}
The combination function is defined as 
\[\comb(\mathbf{x}, z_1, \dots, z_m) = \sigma(\mathbf{x}\mathbf{C} + \mathbf{z}\mathbf{A} + \mathbf{b}),\]
where $\mathbf{x} \in \R^d$ is $\lambda^{l-1}(v)$, $\mathbf{z} = (z_1, \dots, z_m) \in \R^m$ are obtained from template aggregations, and matrices $\mathbf{C} \in \R^{d \times d}, \mathbf{A} \in \R^{m \times d}, \mathbf{b} \in \R^d$ are defined based on $sub(\varphi)$ as follows:

\begin{enumerate}
	\item If $\varphi_k$ is a proposition, then $C_{kk} = 1$,
	\item If $\varphi_k = \neg \varphi_p$, then $C_{pk} = -1$ and $b_k = 1$,
	\item If $\varphi_k = \varphi_p \wedge \varphi_q$, then $C_{pk} = C_{qk} = 1$, and $b_k = -1$,
	\item If $\varphi_k = \varphi_{\iota(j)}$ corresponds to a modal formula of the form $\langle T^{(j)} \rangle^{\ge c_j} (\psi_{j,1}, \dots, \psi_{j, n_j})$, then $A_{jk} = 1$ and $b_k = -c_j + 1$,
\end{enumerate}
and all other entries are $0$. 

Next we show by induction on the structure of subformula $\varphi_k$ that for every graph $G$, every node $v$, and any layer $l \geq \sd(\varphi_k)$, $\lambda^l(v)_k = 1$ iff $(G,v) \models \varphi_k$.

For the base case, if $\varphi_k$ is a proposition, $\sd(\varphi_k) = 0$. By design of layer 0, $\lambda^0(v)_k = 1$ iff $(G,v) \models \varphi_k$, and $\lambda^l(v)_k = 0$ otherwise. As $C_{kk}=1$ preserves identity, $\lambda^l(v)_k = 1$ iff $(G,v) \models \varphi_k$ for all $l \ge 0$.

For the inductive step, assume the statement holds for all subformulae of syntactic depth less than $\varphi_k$. 
\begin{itemize}
	\item Case 1: $\varphi_k = \neg \varphi_p$. 
	By construction, $C_{pk} = -1$ and $b_k = 1$, we have $\lambda^l(v)_k = \sigma(-\lambda^{l-1}(v)_p + 1)$. By induction hypothesis, $\lambda^{l-1}(v)_p = 1$ iff $(G,v) \models \varphi_p$, then $\lambda^l(v)_k = \sigma(0) = 0$. Similarly, $\lambda^{l-1}(v)_p = 0$ iff $(G,v) \not \models \varphi_p$, then $\lambda^l(v)_k = \sigma(1) = 1$. Hence, $\lambda^l(v)_k$ captures $\neg \varphi_p$.
	\item Case 2: $\varphi_k = \varphi_p \wedge \varphi_q$. 
	By construction, $C_{pk} = C_{qk} = 1$ and $b_k = -1$, we have $\lambda^l(v)_k = \sigma( \lambda^{l-1}(v)_p + \lambda^{l-1}(v)_q - 1)$. By induction hypothesis, $\lambda^{l-1}(v)_p = 1$ iff $(G,v) \models \varphi_p$, and $\lambda^{l-1}(v)_q = 1$ iff $(G,v) \models \varphi_q$. Hence, $\lambda^l(v)_k = 1$ iff $(G,v) \models \varphi_p \wedge \varphi_q$, and $\lambda^l(v)_k = 0$ otherwise.
	\item Case 3: $\varphi_k = \varphi_{\iota(j)} = \langle T^{(j)} \rangle^{\ge c_j} (\psi_{j,1}, \dots, \psi_{j, n_j})$.
	Consider first the template aggregation function. By inductive hypothesis, $\lambda^{l-1}(f(i))_{id(j,i)} = 1$ iff $(G, f(i)) \models \psi_i$. The sum in \eqref{eq:aggT} equals to $n$ iff all subformulae $(\psi_{j,1}, \dots, \psi_{j, n_j})$ are satisfied, in which case $\sigma(n - n + 1) = 1$. The sum is $\le n-1$ otherwise, in which case $\sigma$ returns 0.
	Next consider the max-$n$-sum aggregation function bounded by $c_{max}$. Let $N$ be the number of valid embeddings. By construction, $A_{j,k} = 1$ and $b_k = -c + 1$, we have $\lambda^l(v)_k = \sigma(\min(c_{max}, N) - c + 1)$. Since $c \le c_{max}$, if $N \ge c$, then $\min(c_{max}, N) \ge c$, $\lambda^l(v)_k = 1$. If $N < c$, $\lambda^l(v)_k = 0$. Hence, $\lambda^l(v)_k = 1$ iff $N \ge c$ iff $(G,v) \models \varphi_k$, and $\lambda^l(v)_k = 0$ otherwise. \qedhere
\end{itemize}
\end{proof}

By combining Theorems \ref{thm:GNNtoLogic} and \ref{thm:LogictoGNN} we obtain that Boolean bounded $\T$-GNN node classifiers are exactly those that are definable in $\GML(\T)$.

\section{Conclusions and Future Work}
In this paper, we introduced template GNNs as a framework to study the (uniform) expressivity of different graph neural networks in a unified manner. In addition, we introduced the accompanying notions of $\T$-WL algorithm, graded $\T$-bisimulation, and graded modal logic $\GML(\T)$. We showed how various existing approaches to extend the expressivity of GNNs utilising diverse subgraph information can be formalised in our framework.

The main technical result of our paper is a metatheorem stating that, for any finite set of templates $\T$, $\GML(\T)$ captures the uniform expressivity of bounded counting $\T$-GNNs. 
Several existing characterisations, such as the ones by \cite{barcelo2020logical} and \cite{grau2025correspondence} can be seen as special cases of our metatheorem. Similarly, our metatheorem is directly applicable to the $k$-hop subgraph GNNs of \cite{chen2025expressive} and their version of the 1-WL algorithm.

A recent result by \cite{DBLP:journals/corr/abs-2508-06091} (discussed in the introduction) implies that, in general, our logical characterisation does not transfer to the case of non-bounded $\T$-GNNs and $\GML(\T)$, even if restricted to logical classifier definable in first-order logic.
However, since trivially $\T$-GNNs and bounded $\T$-GNNs have the same separation power---on any given graph, $\T$-GNNs aggregate over multisets of bounded cardinality---two graphs are separable by a $\T$-GNN if they are separable by the $\T$-WL algorithm, or equivalently, by a $\GML(\T)$-formula. 

\vspace{2mm}
\noindent\textbf{Open questions and future work.}
We conclude by discussing directions for future research.
\begin{itemize}
    \item Our framework is closely connected to the local graph parameter enabled GNNs of \cite{DBLP:conf/nips/BarceloGRR21}. Their $\mathcal{F}$-MPNNs and $\mathcal{F}$-WL algorithm obtain graph pattern information as our model does, but restricts to standard graph neighbour message passing. What is the precise relationship between these models?
    \item  Can we extend our framework to cover the Hierarchical Ego Graph Neural Networks \citep{soeteman2025logical} by some kind of hybrid extension of our logic.
    \item Can we extend our logical characterisations to cover non-bounded $\T$-GNNs by adding similar counting features to our logic as in \citep{DBLP:conf/icalp/BenediktLMT24,DBLP:journals/theoretics/Grohe24}.
    \item Can we extend our results to cover recursive neural networks using a form of $\mu$-calculus as in \citep{BVVV25}.
\end{itemize}

\bibliographystyle{kr}
\bibliography{refs1}

\begin{thebibliography}{}

\bibitem[\protect\citeauthoryear{Ahvonen \bgroup et al\mbox.\egroup
  }{2024}]{ahvonen2024logical}
Ahvonen, V.; Heiman, D.; Kuusisto, A.; and Lutz, C.
\newblock 2024.
\newblock Logical characterizations of recurrent graph neural networks with
  reals and floats.
\newblock In Globerson, A.; Mackey, L.; Belgrave, D.; Fan, A.; Paquet, U.;
  Tomczak, J.; and Zhan, C., eds., {\em Advances in Neural Information
  Processing Systems}, volume~37,  104205--104249.

\bibitem[\protect\citeauthoryear{Barcel{\'o} \bgroup et al\mbox.\egroup
  }{2020}]{barcelo2020logical}
Barcel{\'o}, P.; Kostylev, E.~V.; Monet, M.; P{\'e}rez, J.; Reutter, J.; and
  Silva, J.-P.
\newblock 2020.
\newblock The logical expressiveness of graph neural networks.
\newblock In {\em 8th International Conference on Learning Representations,
  {ICLR} 2020 Addis Ababa, Ethiopia, April 26-30, 2020}.
\newblock OpenReview.net.

\bibitem[\protect\citeauthoryear{Barcel{\'{o}} \bgroup et al\mbox.\egroup
  }{2021}]{DBLP:conf/nips/BarceloGRR21}
Barcel{\'{o}}, P.; Geerts, F.; Reutter, J.~L.; and Ryschkov, M.
\newblock 2021.
\newblock Graph neural networks with local graph parameters.
\newblock In Ranzato, M.; Beygelzimer, A.; Dauphin, Y.~N.; Liang, P.; and
  Vaughan, J.~W., eds., {\em Advances in Neural Information Processing
  Systems}, volume~34,  25280--25293.

\bibitem[\protect\citeauthoryear{Benedikt \bgroup et al\mbox.\egroup
  }{2024}]{DBLP:conf/icalp/BenediktLMT24}
Benedikt, M.; Lu, C.; Motik, B.; and Tan, T.
\newblock 2024.
\newblock Decidability of graph neural networks via logical characterizations.
\newblock In Bringmann, K.; Grohe, M.; Puppis, G.; and Svensson, O., eds., {\em
  51st International Colloquium on Automata, Languages, and Programming,
  {ICALP} 2024, Tallinn, Estonia, July 8-12, 2024}, volume 297 of {\em LIPIcs},
   127:1--127:20.
\newblock Schloss Dagstuhl - Leibniz-Zentrum f{\"{u}}r Informatik.

\bibitem[\protect\citeauthoryear{Bevilacqua \bgroup et al\mbox.\egroup
  }{2022}]{DBLP:conf/iclr/BevilacquaFLSCB22}
Bevilacqua, B.; Frasca, F.; Lim, D.; Srinivasan, B.; Cai, C.; Balamurugan, G.;
  Bronstein, M.~M.; and Maron, H.
\newblock 2022.
\newblock Equivariant subgraph aggregation networks.
\newblock In {\em 10th International Conference on Learning Representations,
  {ICLR} 2022, Virtual Event, April 25-29, 2022}.
\newblock OpenReview.net.

\bibitem[\protect\citeauthoryear{Bollen \bgroup et al\mbox.\egroup
  }{2025}]{BVVV25}
Bollen, J.; Van~den Bussche, J.; Vansummeren, S.; and Virtema, J.
\newblock 2025.
\newblock {Halting Recurrent GNNs and the Graded $\mu$-Calculus}.
\newblock In {\em {Proceedings of the 22nd International Conference on
  Principles of Knowledge Representation and Reasoning}},  175--184.

\bibitem[\protect\citeauthoryear{Bouritsas \bgroup et al\mbox.\egroup
  }{2023}]{DBLP:journals/pami/BouritsasFZB23}
Bouritsas, G.; Frasca, F.; Zafeiriou, S.; and Bronstein, M.~M.
\newblock 2023.
\newblock Improving graph neural network expressivity via subgraph isomorphism
  counting.
\newblock {\em IEEE Transactions on Pattern Analysis and Machine Intelligence}
  45(1):657--668.

\bibitem[\protect\citeauthoryear{Chen, Zhang, and
  Wang}{2025}]{chen2025expressive}
Chen, Z.; Zhang, Q.; and Wang, R.
\newblock 2025.
\newblock On the expressive power of subgraph graph neural networks for graphs
  with bounded cycles.
\newblock {\em arXiv preprint arXiv:2502.03703}.

\bibitem[\protect\citeauthoryear{Cuenca~Grau, Feng, and
  Wa{\l}{\k{e}}ga}{2026}]{grau2025correspondence}
Cuenca~Grau, B.; Feng, E.; and Wa{\l}{\k{e}}ga, P.~A.
\newblock 2026.
\newblock The correspondence between bounded graph neural networks and
  fragments of first-order logic.
\newblock {\em arXiv preprint arXiv:2505.08021}.
\newblock In AAAI 2026.

\bibitem[\protect\citeauthoryear{de
  Rijke}{2000}]{DBLP:journals/sLogica/Rijke00}
de~Rijke, M.
\newblock 2000.
\newblock A note on graded modal logic.
\newblock {\em Studia Logica} 64(2):271--283.

\bibitem[\protect\citeauthoryear{Frasca \bgroup et al\mbox.\egroup
  }{2022}]{DBLP:conf/nips/FrascaBBM22}
Frasca, F.; Bevilacqua, B.; Bronstein, M.~M.; and Maron, H.
\newblock 2022.
\newblock Understanding and extending subgraph gnns by rethinking their
  symmetries.
\newblock In Koyejo, S.; Mohamed, S.; Agarwal, A.; Belgrave, D.; Cho, K.; and
  Oh, A., eds., {\em Advances in Neural Information Processing Systems},
  volume~35,  31376--31390.

\bibitem[\protect\citeauthoryear{Gilmer \bgroup et al\mbox.\egroup
  }{2017}]{gilmer2017neural}
Gilmer, J.; Schoenholz, S.~S.; Riley, P.~F.; Vinyals, O.; and Dahl, G.~E.
\newblock 2017.
\newblock Neural message passing for quantum chemistry.
\newblock In {\em Proceedings of the 34th International Conference on Machine
  Learning},  1263--1272.

\bibitem[\protect\citeauthoryear{Grohe}{2024}]{DBLP:journals/theoretics/Grohe24}
Grohe, M.
\newblock 2024.
\newblock The descriptive complexity of graph neural networks.
\newblock {\em TheoretiCS} 3.

\bibitem[\protect\citeauthoryear{Hauke and
  Wa{\l}{\k{e}}ga}{2026}]{DBLP:journals/corr/abs-2508-06091}
Hauke, S.~P., and Wa{\l}{\k{e}}ga, P.~A.
\newblock 2026.
\newblock Aggregate-combine-readout gnns are more expressive than logic
  {$C^2$}.
\newblock {\em arXiv preprint arXiv:2508.06091}.
\newblock In AAAI 2026.

\bibitem[\protect\citeauthoryear{Hella \bgroup et al\mbox.\egroup
  }{2015}]{DBLP:journals/dc/HellaJKLLLSV15}
Hella, L.; J{\"{a}}rvisalo, M.; Kuusisto, A.; Laurinharju, J.;
  Lempi{\"{a}}inen, T.; Luosto, K.; Suomela, J.; and Virtema, J.
\newblock 2015.
\newblock Weak models of distributed computing, with connections to modal
  logic.
\newblock {\em Distributed Computing} 28(1):31--53.

\bibitem[\protect\citeauthoryear{Jin \bgroup et al\mbox.\egroup
  }{2024}]{DBLP:conf/icml/JinBCL24}
Jin, E.; Bronstein, M.~M.; Ceylan, {\.I}.~{\.I}.; and Lanzinger, M.
\newblock 2024.
\newblock Homomorphism counts for graph neural networks: All about that basis.
\newblock In {\em Proceedings of the 41st International Conference on Machine
  Learning},  22075--22098.

\bibitem[\protect\citeauthoryear{Linial}{1992}]{DBLP:journals/siamcomp/Linial92}
Linial, N.
\newblock 1992.
\newblock Locality in distributed graph algorithms.
\newblock {\em SIAM Journal on Computing} 21(1):193--201.

\bibitem[\protect\citeauthoryear{Morris \bgroup et al\mbox.\egroup
  }{2019}]{morris2019weisfeiler}
Morris, C.; Ritzert, M.; Fey, M.; Hamilton, W.~L.; Lenssen, J.~E.; Rattan, G.;
  and Grohe, M.
\newblock 2019.
\newblock Weisfeiler and leman go neural: Higher-order graph neural networks.
\newblock In {\em Proceedings of the AAAI Conference on Artificial
  Intelligence}, volume~33,  4602--4609.

\bibitem[\protect\citeauthoryear{Otto}{2019}]{otto2019graded}
Otto, M.
\newblock 2019.
\newblock Graded modal logic and counting bisimulation.
\newblock {\em arXiv preprint arXiv:1910.00039}.

\bibitem[\protect\citeauthoryear{Pflueger, Tena~Cucala, and
  Kostylev}{2024}]{pflueger2024recurrent}
Pflueger, M.; Tena~Cucala, D.; and Kostylev, E.~V.
\newblock 2024.
\newblock Recurrent graph neural networks and their connections to bisimulation
  and logic.
\newblock In {\em Proceedings of the AAAI Conference on Artificial
  Intelligence}, volume~38,  14608--14616.

\bibitem[\protect\citeauthoryear{Sato, Yamada, and
  Kashima}{2019}]{DBLP:conf/nips/SatoYK19}
Sato, R.; Yamada, M.; and Kashima, H.
\newblock 2019.
\newblock Approximation ratios of graph neural networks for combinatorial
  problems.
\newblock In Wallach, H.~M.; Larochelle, H.; Beygelzimer, A.;
  d'Alch{\'{e}}{-}Buc, F.; Fox, E.~B.; and Garnett, R., eds., {\em Advances in
  Neural Information Processing Systems}, volume~32,  4083--4092.

\bibitem[\protect\citeauthoryear{Scarselli \bgroup et al\mbox.\egroup
  }{2009}]{DBLP:journals/tnn/ScarselliGTHM09}
Scarselli, F.; Gori, M.; Tsoi, A.~C.; Hagenbuchner, M.; and Monfardini, G.
\newblock 2009.
\newblock The graph neural network model.
\newblock {\em IEEE Transactions on Neural Networks} 20(1):61--80.

\bibitem[\protect\citeauthoryear{Soeteman and ten
  Cate}{2025}]{soeteman2025logical}
Soeteman, A., and ten Cate, B.
\newblock 2025.
\newblock Logical expressiveness of graph neural networks with hierarchical
  node individualization.
\newblock {\em arXiv preprint arXiv:2506.13911}.
\newblock In NeurIPS 2025.

\bibitem[\protect\citeauthoryear{Tena~Cucala and
  Cuenca~Grau}{2024}]{DBLP:conf/kr/CucalaG24}
Tena~Cucala, D.~J., and Cuenca~Grau, B.
\newblock 2024.
\newblock Bridging max graph neural networks and datalog with negation.
\newblock In {\em Proceedings of the 21st International Conference on
  Principles of Knowledge Representation and Reasoning},  950--961.

\bibitem[\protect\citeauthoryear{Tena~Cucala \bgroup et al\mbox.\egroup
  }{2023}]{DBLP:conf/kr/CucalaGMK23}
Tena~Cucala, D.~J.; Cuenca~Grau, B.; Motik, B.; and Kostylev, E.~V.
\newblock 2023.
\newblock On the correspondence between monotonic max-sum gnns and datalog.
\newblock In {\em Proceedings of the 20th International Conference on
  Principles of Knowledge Representation and Reasoning},  658--667.

\bibitem[\protect\citeauthoryear{Weisfeiler and
  Leman}{1968}]{leman1968reduction}
Weisfeiler, B., and Leman, A.
\newblock 1968.
\newblock A reduction of a graph to a canonical form and an algebra arising
  during this reduction.
\newblock {\em Nauchno-Technicheskaya Informatsiya} 2(9):12--16.

\bibitem[\protect\citeauthoryear{Xu \bgroup et al\mbox.\egroup
  }{2019}]{xu2019powerful}
Xu, K.; Hu, W.; Leskovec, J.; and Jegelka, S.
\newblock 2019.
\newblock How powerful are graph neural networks?
\newblock In {\em 7th International Conference on Learning Representations,
  {ICLR} 2019, New Orleans, LA, USA, May 6-9, 2019}.
\newblock OpenReview.net.

\end{thebibliography}

\end{document}